\documentclass{article}
\usepackage{iclr2027_conference,times}

\usepackage{amsmath}
\usepackage{amssymb}
\usepackage{amsthm}
\usepackage{graphicx}
\usepackage{booktabs}
\usepackage{microtype}
\usepackage[font=small,labelfont=bf,skip=4pt]{caption}
\usepackage{hyperref}
\usepackage{url}

\hypersetup{colorlinks=true, linkcolor=black, citecolor=black, urlcolor=black,
            filecolor=black, breaklinks=true}

\newcommand{\eps}{\varepsilon}
\newcommand{\epslo}{\eps_{\mathrm{lo}}}
\newcommand{\epsup}{\eps_{\mathrm{up}}}
\newcommand{\ci}[2]{{\scriptsize$[#1,\,#2]$}}
\graphicspath{{figs/pdf/}}

\newtheorem{theorem}{Theorem}
\newtheorem{proposition}[theorem]{Proposition}
\newtheorem{corollary}[theorem]{Corollary}
\newtheorem{lemma}[theorem]{Lemma}
\theoremstyle{remark}

\title{When Do Options Help? Policy Necrosis and\\Redundant Coverage in Option-Critic}

\author{%
Bingyun Liu\thanks{Equal contribution.\ $^\dagger$Corresponding author.} \\
Institute of Automation\\
Chinese Academy of Sciences\\
\texttt{liubingyun2021@ia.ac.cn} \\
\And
Yuheng Jing$^{*\dagger}$ \\
Institute of Automation\\
Chinese Academy of Sciences\\
\texttt{jingyuheng2022@ia.ac.cn}
}

\iclrfinalcopy

\begin{document}

\maketitle
\lhead{}
\renewcommand{\headrulewidth}{0pt}

\begin{abstract}
Option-critic learns options: sub-policies together with a learned rule for when
each one hands control back. Its headline result is that performance improves as
options are added. We explain that result, with theory and experiment. First, the
termination rule option-critic learns by maximising return contributes nothing.
When the termination test and the policy that picks options read the same values,
the test fires at every step, so the learned rule is identical to always
terminating. When that policy explores and the test does not, as in option-critic
itself, the rule can block the exploration; there are instances where it suffers
$\Omega(T)$ regret while always terminating holds to $O(\log T)$. Forcing
termination at every step leaves the option-count curve intact. Second, the
policy inside an option barely explores at all, so a state locks onto the first
action that looked good and never updates again. We name this policy necrosis,
give a state-level test for it, and find three fifths of states necrotic in a
typical option. Restoring exploration repairs those states, and one option then
solves the task. Third, extra options improve no option; what falls is the chance
that all of them fail in the same state, from $59\%$ to $4\%$, and performance
follows that joint quantity.
\end{abstract}

\section{Introduction}

The options framework gives reinforcement learning agents temporally extended
actions: an option is an intra-option policy, an initiation set, and a
termination condition \citep{sutton1999between,precup2000temporal}. Option-critic
made the framework practical by differentiating the discounted return with
respect to both the intra-option policy parameters and the termination
parameters, so that the control objective itself shapes the options instead of a
designer handing them to the agent \citep{bacon2017option}. Its tabular
demonstration is FourRooms, and the demonstration is a scaling result: with more
options the agent reaches the goal faster, and the effect grows with the
difficulty of the dynamics. That result has been reproduced, extended to
continuous control, and built upon many times over
\citep{klissarov2017ppoc,harb2018waiting,riemer2018abstract,zhang2019dac,khetarpal2020interest}.

What the scaling result licenses is less clear than it looks. Three things move
together when the number of options $K$ grows. Options are temporally extended,
so a larger option set changes the temporal structure of behaviour. Options are
separate function approximators, so a larger option set changes the capacity and
the optimisation geometry of the lower level. And options are chosen by a learned
router, so a larger option set changes how much of the decision burden sits above
the primitive actions. A curve of performance against $K$ cannot separate them,
and reading it as evidence for temporal abstraction assumes the other two do
nothing.

This paper separates them, and the answer is not temporal abstraction. We
manipulate one factor at a time --- the termination rule, the lower-level
exploration rate, and $K$ --- across several transition kernels, two reward
definitions, and a tabular and a neural implementation, and we read performance
off frozen checkpoints under a protocol that gives every condition the same
zero-exploration test policy. Figure~\ref{fig:framework} lays out the
environment, the agent and the hypothesis; three results follow, and together
they form a single account of what option-critic is doing.

\textbf{The termination function that the termination gradient learns is not a
source of performance.} Our target is specific: the termination rule that
option-critic obtains by maximising return, whose fixed point terminates exactly
when continuing is worth less than handing control back. We prove that this rule
is either vacuous or harmful. When the termination test and the option selector
read the same values, the test fires at every state, so the learned rule is
pathwise identical to $\beta\equiv1$ and adds nothing. When the selector explores
while the critic feeding the test does not, the rule can veto the exploration the
selector asks for; we exhibit instances on which the return-maximising rule
suffers $\Omega(T)$ regret, while $\beta\equiv1$ holds to $O(\log T)$ on every
instance of the class. The experiments match the theory. Forcing $\beta\equiv1$, which
destroys temporal extension outright, leaves the whole dependence on $K$
untouched, and the measured effect of learned termination reverses sign as $K$
grows.

\textbf{Option-critic's lower level suffers a state-local learning failure, and
we can detect it.} The original algorithm runs its intra-option policies as a
softmax at temperature $10^{-3}$ with no lower-level exploration, and with
learning rate $0.25$ that multiplies every advantage by $250$ before it reaches
the logits. A state commits to whichever action was first sampled with a positive
advantage; the score function then vanishes, and the state never updates again.
We call this \emph{policy necrosis}, and we give a state-level test for it that
needs only the set of actions an optimal policy would accept. Roughly half of all
states are necrotic in a typical option. Restarting lower-level exploration from
a shared checkpoint repairs them: a single option goes from $306.7$ to $55.4$
expected steps under identical zero-exploration evaluation, and improves on all
$350$ seeds. Two different exploration mechanisms reproduce the repair in the
neural implementation.

\textbf{What extra options provide is redundant coverage.} Adding options repairs no
individual option --- the mean single-option necrosis rate holds near $60\%$ at
every $K$ --- but the chance that \emph{every} option is necrotic in the same
state falls from $59\%$ to $4\%$, and performance follows that joint quantity and
not the individual one. Necrotic sets overlap more than independence predicts and
far less than a common cause would, so a pool of $K$ options is worth about $K/2$
independent attempts at covering the state space. Once joint coverage exists the
residual error moves upstairs, to a router that picks a necrotic option while a
viable one is available. And once exploration has repaired the lower level,
further options stop returning anything and start costing.

The three results compose. Necrosis is the failure, option multiplicity is a
workaround that hides it, lower-level exploration is a repair that removes it,
and the termination gradient plays no part. Beyond the specific verdict on
option-critic, the paper contributes a diagnostic that is cheap wherever an
acceptable-action set is available, a theoretical account of why a
return-maximising termination rule cannot be a source of gains, and an evaluation
protocol --- freeze, disable exploration everywhere, solve exactly --- under
which claims about what a hierarchical agent has learned can be settled instead
of argued.

\begin{figure}[t]
\centering
\includegraphics[width=\textwidth]{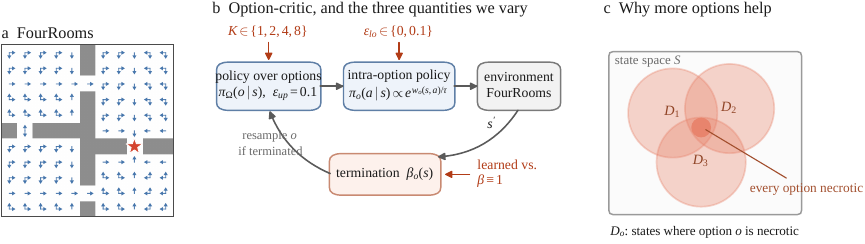}
\caption{\textbf{The benchmark, the agent, and the hypothesis.}
(a) FourRooms with the goal in the east doorway (star) and, for every non-goal
cell, the complete set of actions that reduce the wall-respecting distance to the
goal; $65$ of the $103$ cells admit two. (b) The option-critic loop and the three quantities we manipulate; everything else stays at the values of the original work. (c) An option is necrotic in a state when its policy has
committed there to an action no optimal policy would take, so option multiplicity
acts on $\bigcap_o D_o$ and not on any individual $D_o$.}
\label{fig:framework}
\end{figure}

\section{Setup}
\label{sec:setup}

\paragraph{Option-critic.}
An option $o$ consists of an intra-option policy $\pi_o(a\,|\,s)$ and a
termination function $\beta_o(s)$; a policy over options $\pi_\Omega(o\,|\,s)$
picks the next option whenever the current one terminates
\citep{sutton1999between}. Option-critic learns all three from the return
\citep{bacon2017option}. In the tabular implementation we study, $\pi_o$ is a
softmax over preference weights, $\pi_o(a\,|\,s)\propto\exp(w_o(s,a)/\tau)$, and
receives the update
\begin{equation}
\label{eq:intra}
w_o(s,\cdot)\;\leftarrow\;w_o(s,\cdot)+\alpha\,\big(Q_U(s,o,a)-Q_\Omega(s,o)\big)\,
\big(e_a-\pi_o(\cdot\,|\,s)\big),
\end{equation}
where $Q_U(s,o,a)$ is the critic's value for taking primitive action $a$ under
option $o$ at $s$ and $Q_\Omega(s,o)$ its value for running option $o$ from $s$,
so their difference is the intra-option advantage, and $e_a$ is the one-hot vector
on the executed action, making $e_a-\pi_o(\cdot\,|\,s)$ the softmax score
function. The termination function $\beta_o$ is a sigmoid updated along the
option advantage $A_\Omega(s,o)=Q_\Omega(s,o)-\max_{o'}Q_\Omega(s,o')$, and
$\pi_\Omega$ is $\epsup$-greedy on the tabular $Q_\Omega$. The original hyperparameters are
$\tau=10^{-3}$, $\alpha=0.25$ for both the intra-option and termination updates,
$0.5$ for the critic, $\epsup=0.1$, and discount $0.99$. Every reported
configuration keeps them.

\paragraph{Environments.}
FourRooms is the benchmark option-critic was introduced on, and it is also an
unusually good instrument. Its state space is finite and small enough to
enumerate, so for any transition kernel we can solve the MDP exactly, read off
the set of actions an optimal policy would accept in every state, and turn that
set into a per-state test of whether an option is usable there; nothing in
Section~\ref{sec:diagnostics} is available without it. The grid is $13\times13$
with $104$ free cells and the goal in the east doorway, leaving $103$ non-goal
states. Under \emph{local slip} the agent moves to a uniformly random
neighbouring free cell with probability $p$; under \emph{global teleport} it
jumps to a uniformly random free cell with probability $p$ and slips locally with
probability $(1-p)p$, with $p\in\{0,\tfrac13,\tfrac12\}$. Two rewards are used:
the original \emph{sparse} reward of $1$ on entering the goal, and a
\emph{constant cost} of $-1$ on every transition, which makes the optimal policy
the shortest path and the acceptable action set purely combinatorial. Episodes
are capped at $1000$ steps and training runs for $1000$ episodes.

Two continuous variants test whether the mechanism survives function
approximation, where no state can be enumerated and the diagnostics do not apply.
In \emph{Direct} the actions are unit displacements with uniform noise and the
observation is the position. \emph{Inertia} changes both the state and the action
spaces: actions are accelerations, the observation carries a velocity, and a
collision zeroes it, so the agent controls a second-order system on the same
floor plan. All four option-critic components become two-layer tanh networks
trained with Adam, target networks and a replay buffer;
Appendix~\ref{app:setup} lists every hyperparameter.

\paragraph{Evaluation.}
Training curves conflate what a policy has learned with how much it randomises
while acting, the exact confound at stake here, so we report performance from
frozen checkpoints. Parameters are copied out and never updated; explicit
lower-level exploration is set to zero in every condition, so a branch trained
with $\epslo=0.1$ is evaluated with $\epslo=0$; $\epsup=0.1$ is retained and
marginalised exactly; and the expected number of capped transitions to the goal
is computed by finite-horizon survival dynamic programming on the true kernel
with horizon $1000$, starting uniformly over the $103$ non-goal states. This
evaluation has no sampling error, so all reported uncertainty comes from the
$350$ matched training seeds, and every contrast is paired within seed. The
neural experiments and the tabular termination grid instead use the mean number
of steps over the last $100$ training episodes, the metric the original scaling
result was reported on, and every figure and table that uses it says so.
Appendix~\ref{app:stats} gives the resampling scheme behind every interval.

\section{Detecting policy necrosis}
\label{sec:diagnostics}

The measurements below need a notion of what it means for one option to be usable
in one state, and a way to attribute a failure either to the options or to the
router that picks among them.

\paragraph{Necrotic options.}
Fix a frozen checkpoint and a state $s$, and let $A^\star(s)$ be the set of
primitive actions an optimal policy for the true kernel would be willing to take
at $s$. Option $o$ is \emph{necrotic} at $s$ when its intra-option policy has
committed there to something outside that set. Under the constant cost this needs
no threshold: $A^\star(s)$ is exactly the set of actions that reduce the
wall-respecting distance to the goal, all ties kept ($65$ of $103$ states have
two), and $o$ is viable at $s$ when $\arg\max_a w_o(s,a)\in A^\star(s)$. Under the
sparse reward we take $A^\star(s)$ from value iteration on the exact kernel,
keeping all numerically tied maximisers, and call $o$ necrotic when
$\sum_{a\in A^\star(s)}\pi_o(a\,|\,s)<0.01$; the deterministic case agrees state
by state with a breadth-first search. Both criteria weight all states equally and
are evaluated with exploration switched off.

\paragraph{Where the failure sits.}
With $K$ options, every (state, seed) pair falls in exactly one of three classes.
It is \textbf{D} when all $K$ options are necrotic, so no choice the router could
make would help; \textbf{M} when at least one option is viable but the greedy
$\arg\max_o Q_\Omega(s,o)$ selects a necrotic one, so usable coverage exists and
is missed; and \textbf{R} when the router selects a viable option. Multiplicity
can act on $\Pr(\mathrm{D})=\Pr(\bigcap_o D_o)$, while the mean single-option
necrosis rate $\tfrac1K\sum_o\Pr(D_o)$ says whether individual options improved.
Separating them is the point: if extra options make each option better, both fall
together; if extra options supply redundancy, only the first falls.

\section{The learned termination function is not a source of performance}
\label{sec:termination}

\subsection{What the termination gradient is heading for}
\label{sec:theory}

Option-critic updates $\beta_o$ along the option advantage
$A_\Omega(s,o)=Q_\Omega(s,o)-\max_{o'}Q_\Omega(s,o')$, which is never positive, so
the update only ever raises $\beta_o(s)$, wherever the current option is not the
greedy one. The rule it heads for is greedy interruption,
\begin{equation}
\label{eq:hard}
\beta(s,o)=\begin{cases}0,&Q_\Omega(s,o)>V_\Omega(s),\\ 1,&\text{otherwise,}\end{cases}
\end{equation}
and the collapse of learned $\beta$ towards $1$ that the literature reports
\citep{harb2018waiting,harutyunyan2019termination} is that drift. We analyse
\eqref{eq:hard} in a two-level setting where an upper level picks an option and a
lower level picks an action inside it, both driven by value estimates.
Appendix~\ref{app:theory} gives the model, the algorithms and the proofs.

\begin{proposition}[Vacuity under value consistency]
\label{prop:vacuous}
If the selector and the termination test read the same estimates $Q_t$ and the
state value obeys $V_t(s)=\max_oQ_t(s,o)$, then $Q_t(s,C)\le V_t(s)$ for every
current option $C$, so \eqref{eq:hard} terminates at every step and produces
trajectories and regret identical to $\beta\equiv1$.
\end{proposition}

\begin{theorem}[Worst-case separation under the standard definitions]
\label{thm:separation}
If instead the selector uses optimistic indices while the test reads the critic's
raw values, which is option-critic's own definitions, then for two-level UCB there
is an instance with two options and two actions each on which, writing
$\mathrm{gi}$ for rule~\eqref{eq:hard}, for every $T\ge4$,
\[
\mathcal R_T^{\mathrm{gi}}\;\ge\;\frac{T-2}{64},
\qquad
\mathcal R_T^{\beta\equiv1}\;\le\;64\ln T+\tfrac12\bigl(1+\tfrac{\pi^2}{3}\bigr).
\]
Over the class the worst case is $\Theta(T)$ for $\mathrm{gi}$ against
$O(\sqrt{MT\ln T})$ for $\beta\equiv1$, with $M$ the number of option--action
pairs, and a two-step tabular MDP that reselects the option every episode still
forces $\Omega(T)$.
\end{theorem}

\begin{theorem}[Removing termination restores a guarantee]
\label{thm:ucbvi}
Under $\beta\equiv1$ the two-level agent is isomorphic to a flat agent on the MDP
whose actions are the pairs $(o,a)$, so UCBVI-BF gives, with probability at least
$1-\delta$ and $L=\ln(5HSDT/\delta)$,
$\operatorname{Reg}_N\le30L\sqrt{HSDT}+2500H^2S^2DL^2+4H\sqrt{TL}
=\widetilde O(\sqrt{HSDT})$, where $S$ counts states, $D$ option--action pairs
per state, $H$ the horizon and $T=NH$.
\end{theorem}

A termination function trained to maximise return therefore has no room in which
to help. Consistent values make it an identity operation; inconsistent values
make it an exploration veto with no uniform sublinear guarantee, and the veto
costs most exactly when the option set grows and the selector has more to
explore. The separation is worst case and not pointwise --- the veto can also
suppress exploration of a genuinely worse option, and Appendix~\ref{app:reverse}
gives an instance where it wins --- but a component whose sign depends on the
instance is not a source of reliable gains either way.

\subsection{What the experiments measure}

\begin{table}[t]
\centering
\caption{\textbf{The measured effect of learned termination reverses as the
option set grows.} Sparse-reward FourRooms, the setting option-critic was
introduced on. Each entry is $100\times(\text{Term}-\beta{\equiv}1)/\beta{\equiv}1$
on the mean step count over the last $100$ training episodes, so negative favours
learned termination, with $95\%$ paired bootstrap intervals. The tabular grid
resamples $10$ rank blocks of $35$ seeds per cell, the neural grid $50$ individual
seeds. $K{=}1$ is omitted because a single option cannot change identity on
termination.}
\label{tab:termination}
\small
\setlength{\tabcolsep}{4pt}
\begin{tabular}{lccc}
\toprule
& \multicolumn{3}{c}{Relative effect (\%) with 95\% CI} \\
\cmidrule(lr){2-4}
Environment & $K=2$ & $K=4$ & $K=8$ \\
\midrule
\multicolumn{4}{l}{\emph{Tabular}} \\
No perturbation      & $-23.8$ \ci{-37.0}{-7.6}  & $+0.7$ \ci{-12.0}{13.8} & $+11.1$ \ci{7.6}{14.2} \\
Local slip $p=1/3$   & $-11.6$ \ci{-21.3}{-0.8}  & $-0.9$ \ci{-11.3}{10.1} & $+29.9$ \ci{22.1}{36.3} \\
Local slip $p=1/2$   & $-18.8$ \ci{-26.2}{-10.3} & $+4.8$ \ci{1.0}{9.0}    & $+37.0$ \ci{33.7}{40.7} \\
Teleport $p=1/3$     & $-20.0$ \ci{-28.6}{-10.6} & $+1.8$ \ci{-4.1}{7.3}   & $+29.9$ \ci{24.5}{34.6} \\
Teleport $p=1/2$     & $-14.6$ \ci{-25.0}{-3.6}  & $+5.0$ \ci{0.3}{11.2}   & $+20.6$ \ci{17.0}{24.3} \\
\midrule
\multicolumn{4}{l}{\emph{Neural}} \\
Direct               & $+5.3$ \ci{0.4}{11.3}     & $+0.3$ \ci{-21.4}{28.0} & $-29.0$ \ci{-59.0}{20.2} \\
Inertia              & $-1.0$ \ci{-3.4}{1.3}     & $+0.2$ \ci{-13.1}{15.9} & $+16.4$ \ci{-18.2}{65.0} \\
\bottomrule
\end{tabular}
\end{table}

Forcing $\beta\equiv1$ removes temporal extension outright: the agent draws a new
option after every primitive step, so an option is nothing more than a latent
index refreshed at every state. Temporal extension is the only thing the control
takes away, so a scaling result that rested on it should collapse here.

It does not. Table~\ref{tab:termination} gives the measured effect in the
sparse-reward setting option-critic was introduced on, on the metric the original
scaling result was reported on. The two conditions do differ, and they differ in
the way Theorem~\ref{thm:separation} predicts: the sign tracks how much
exploration the selector has to do. In the tabular grid learned termination wins
all five kernels at $K{=}2$ by $11.6\%$ to $23.8\%$, with every interval below
zero; it is within $\pm5\%$ in all five at $K{=}4$, with three intervals covering
zero; and it loses all five at $K{=}8$ by $11.1\%$ to $37.0\%$, with every
interval above zero. With two options there is little for the veto to block; with
eight, the router has a large candidate set and blocking it costs. The neural
grid gives no consistent direction, with a $29.0\%$ advantage in Direct at
$K{=}8$ against a $16.4\%$ disadvantage in Inertia, and intervals wide enough at
$K{=}4$ and $K{=}8$ to admit either sign. A component whose measured contribution
reverses with a hyperparameter it is meant to be orthogonal to is not a stable
source of gains, and the option-count scaling it is supposed to explain survives
its removal.

The shortest-path variant lets us make the comparison without relying on training
curves at all, because both arms can be evaluated by exact frozen dynamic
programming. There the paired difference between learned option-critic and its
$\beta\equiv1$ match contains zero at every option count, and the two agents
trace the same eightfold improvement as $K$ grows,
$962.5\rightarrow821.4\rightarrow377.2\rightarrow117.4$ expected steps against
$963.2\rightarrow822.3\rightarrow384.4\rightarrow122.5$. Whatever produces that
improvement is present when options have no temporal extent at all.
Table~\ref{tab:termination-frozen} in Appendix~\ref{app:extra} gives those
numbers, and Figure~\ref{fig:termination} there redraws
Table~\ref{tab:termination} as a heat map.

\section{Policy necrosis, and what repairs it}
\label{sec:exploration}

\paragraph{Why states stop learning.}
Look again at the intra-option update in Eq.~\ref{eq:intra}. The quantity that
enters the policy is $w_o(s,a)/\tau$, so with $\tau=10^{-3}$ and $\alpha=0.25$ an
advantage is multiplied by $\alpha/\tau=250$ before it reaches the logits. From a
uniform policy over four actions, an advantage of $0.008$ already leaves the
sampled action with probability above $0.7$, and $0.032$ pushes it past $0.999$.
The commitment is self-sealing: the score function $e_a-\pi_o(\cdot\,|\,s)$ that
multiplies the advantage goes to zero for the committed action, and the
alternatives that carry a non-vanishing score are sampled with probability close
to zero. Under a sparse reward the first advantage a state sees comes from an
essentially uninformed critic, so what it commits to is close to arbitrary, and
further training on the same option does not repair it. This is the softmax
gravity well described by \citet{mei2020escaping}, with the temperature turning a
slow drift into a one-update event.

\paragraph{A matched fork isolates the repair.}
To test the account we train a common $500$-episode prefix with $\epslo=0$,
snapshot the entire learner including the random number generator, and continue
along two branches: a control that keeps $\epslo=0$ and a rescue branch that
mixes in uniform random actions with probability $\epslo=0.1$. The rescue branch
changes only which actions are executed, so nothing about the learning rule
changes. Both branches are evaluated with $\epslo=0$, so any difference comes
from parameters written during training and not from randomness at test time;
the episode-500 checkpoints of the two branches are bit-for-bit identical.

\begin{figure}[t]
\centering
\includegraphics[width=\textwidth]{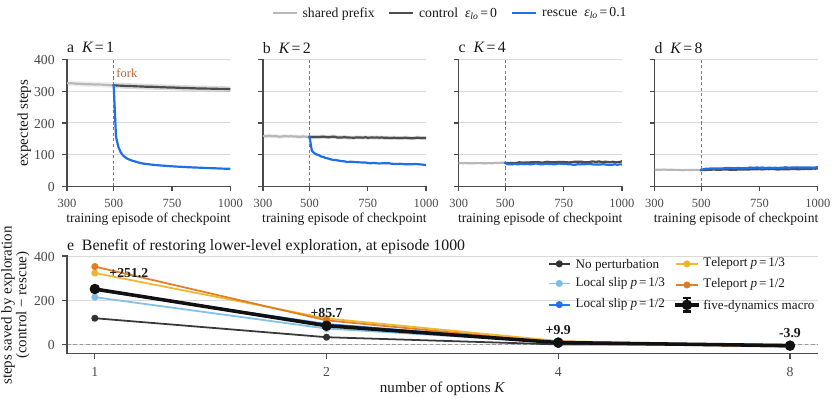}
\caption{\textbf{Restoring lower-level exploration replaces the entire benefit of
extra options.} (a--d) Frozen exact evaluation of every checkpoint, five-kernel
average, per option count. Training shares a $\epslo=0$ prefix through episode
$500$ and then forks; both branches are evaluated with $\epslo=0$. Bands are
$95\%$ bootstrap intervals over $350$ matched seeds. (e) Paired step reduction at
episode $1000$, per kernel and averaged.}
\label{fig:exploration}
\end{figure}

\begin{table}[t]
\centering
\caption{\textbf{Repairing the lower level removes the work that extra options
were doing.} Left: sparse-reward tabular FourRooms, five-kernel average of
expected capped transitions at the episode-1000 checkpoint, paired over $350$
seeds. Right: constant-cost FourRooms, $\beta\equiv1$ agents trained with and
without lower-level exploration, evaluated identically. All evaluations use
$\epslo=0$.}
\label{tab:exploration}
\small
\setlength{\tabcolsep}{4.5pt}
\begin{tabular}{cccccc@{\hspace{14pt}}cc}
\toprule
& \multicolumn{5}{c}{Sparse reward: matched fork at episode 500} & \multicolumn{2}{c}{Constant cost} \\
\cmidrule(lr){2-6}\cmidrule(lr){7-8}
$K$ & control & rescue & reduction & 95\% CI & seeds helped & train $\epslo=0$ & train $\epslo=0.1$ \\
\midrule
1 & 306.66 & 55.42 & $251.23$ & $[242.59,\ 259.99]$ & 350/350 & 963.21 & \textbf{36.50} \\
2 & 152.93 & 67.18 & $85.75$  & $[79.68,\ 91.98]$   & 344/350 & 822.31 & 45.21 \\
4 & 79.14  & 69.28 & $9.86$   & $[6.74,\ 13.03]$    & 211/350 & 384.37 & 55.09 \\
8 & 56.42  & 60.35 & $-3.94$  & $[-5.73,\ -2.14]$   & 138/350 & 122.46 & 85.03 \\
\bottomrule
\end{tabular}
\end{table}

\paragraph{One option is enough.}
Figure~\ref{fig:exploration}a shows what happens with a single option. The
control drifts from $318.7$ at the fork to $306.7$ expected steps by the end of
training; the rescue branch is at $156.1$ ten episodes later, $76.6$ by episode
$600$, and $55.4$ at the end. All $350$ seeds improve, and the paired reduction
is $251.2$ steps \ci{242.59}{259.99}, running from $119.6$ under no perturbation
to $352.3$ under teleport with $p=\tfrac12$ (Figure~\ref{fig:exploration}e).
Under the constant-cost reward the same intervention takes a single
$\beta\equiv1$ option from $963.2$ to $36.5$ expected steps and its frozen
success probability from $0.037$ to $0.999$. One option, evaluated with no
exploration at all, solves the task.

\paragraph{Why the benefit shrinks as $K$ grows.}
Panels b--d of Figure~\ref{fig:exploration} repeat the fork at $K=2,4,8$, and the
gap the rescue branch opens shrinks each time. The paired reduction falls from
$251.2$ steps at $K{=}1$ to $85.7$, then $9.9$, and at $K{=}8$ it turns negative
at $-3.9$ \ci{-5.73}{-2.14}, with only $138$ of $350$ seeds improving
(Table~\ref{tab:exploration}). Under the constant-cost reward the exploring agent
gets steadily worse as options are added, from $36.5$ to $85.0$ expected steps.
The two interventions are not interchangeable and do not act on the same object.
Exploration acts on the cause, restoring the sampling that lets a committed state
recover, and Section~\ref{sec:redundancy} shows it lowers the single-option
necrosis rate directly; extra options leave every individual policy where it was
and act on the aggregate, by making it unlikely that all of them fail in the same
place. The shrinking gap records the interaction: multiplicity hides most of the
cost of necrosis, leaving a repair little headroom, and once the repair has
removed the necrosis multiplicity has nothing to compensate for.

\paragraph{The neural implementation agrees, through two different mechanisms.}
Figure~\ref{fig:neural} repeats the comparison in continuous FourRooms with
$\beta\equiv1$ throughout, contrasting the base softmax against explicit
$\eps$-greedy behaviour and against entropy-regularised sampling. With one option
the base softmax sits at the $1000$-step cap in both environments ($977$ and
$970$), $\eps$-greedy brings it to $299$ and $250$, and entropy regularisation to
$207$ and $156$. Adding options helps the base softmax a great deal, helps
$\eps$-greedy much less, and makes the entropy configuration monotonically worse:
$207\rightarrow225\rightarrow244\rightarrow260$ in Direct and
$156\rightarrow158\rightarrow185\rightarrow221$ in Inertia. Inertia runs at
$\tau=0.5$, five hundred times the tabular temperature, and its one-option base
softmax still fails, so the failure is no artefact of the extreme tabular
setting.

\begin{figure}[t]
\centering
\includegraphics[width=\textwidth]{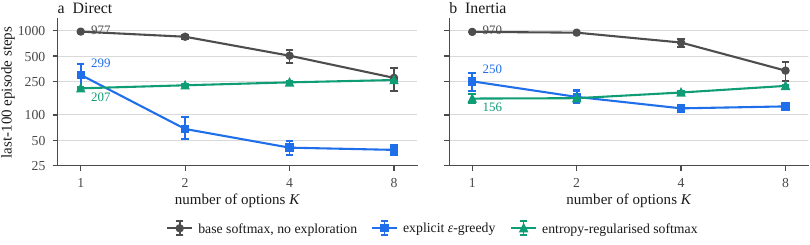}
\caption{\textbf{Two exploration mechanisms make one option sufficient in the
neural implementation.} Mean steps over the last $100$ training episodes with
$\beta\equiv1$ throughout, $50$ seeds per point, $95\%$ bootstrap intervals, log
scale. Numbers mark the one-option values.}
\label{fig:neural}
\end{figure}

\section{What extra options provide is redundant coverage}
\label{sec:redundancy}

If the lower level is necrotic and extra options carry the load instead, what
exactly is the load they carry? Measuring that calls for the constant-cost
reward, where the optimal policy is the shortest path, so the acceptable action
set at a state is exactly the set of distance-reducing actions and the necrosis
test is combinatorial, with no threshold and no dependence on a value scale.
Under the sparse reward the value function encodes reachability and the numerical
spread between distinct reachable routes is small, which makes a shortest-path
yardstick a coarser instrument there. We therefore read the coverage story off
the constant-cost grid, and Figure~\ref{fig:sparse-coverage} in
Appendix~\ref{app:extra} reports the sparse-reward version, which agrees.

\begin{figure}[t]
\centering
\includegraphics[width=\textwidth]{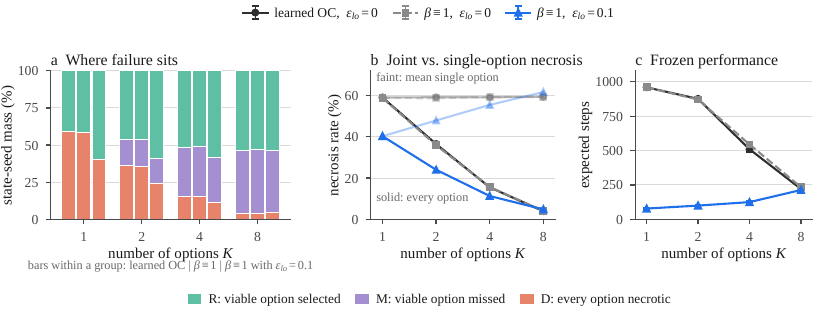}
\caption{\textbf{More options make simultaneous failure rare without making any
option better.} Constant-cost FourRooms under teleport $p=\tfrac12$,
episode-1000 checkpoints, $350$ matched seeds, threshold-free necrosis test.
(a) Every (state, seed) pair in exactly one class. (b) The rate at which every
option is necrotic (solid) against the mean single-option rate (faint). (c) Exact
frozen performance of the same checkpoints. Learned option-critic and its
$\beta\equiv1$ match coincide throughout.}
\label{fig:coverage}
\end{figure}

\begin{table}[t]
\centering
\caption{\textbf{Going from one option to eight, within seed.} Constant-cost
FourRooms, teleport $p=\tfrac12$, episode-1000 checkpoints, $350$ matched seeds.
Necrosis rates are in percentage points and steps are expected capped
transitions; negative means the eight-option agent is better. Joint necrosis
collapses in every row while single-option necrosis does not improve in any of
them.}
\label{tab:redundancy}
\small
\setlength{\tabcolsep}{4pt}
\begin{tabular}{lccc}
\toprule
Agent & $\Delta$ all necrotic & $\Delta$ single-option necrosis & $\Delta$ expected steps \\
\midrule
learned OC, train $\epslo=0$      & $-55.11$ \ci{-55.65}{-54.57} & $+0.31$ \ci{-0.21}{0.83}  & $-739.0$ \ci{-749.6}{-728.1} \\
$\beta\equiv1$, train $\epslo=0$  & $-54.66$ \ci{-55.22}{-54.10} & $+0.57$ \ci{0.02}{1.12}   & $-726.6$ \ci{-739.7}{-713.2} \\
$\beta\equiv1$, train $\epslo=0.1$& $-35.49$ \ci{-35.98}{-34.98} & $+21.39$ \ci{20.90}{21.88} & $+134.9$ \ci{125.4}{144.8} \\
\bottomrule
\end{tabular}
\end{table}

For an agent that never explored, the mean single-option necrosis rate is
$59.04\%$ at $K{=}1$ and $59.35\%$ at $K{=}8$, and it never leaves that band. Close
to three fifths of the state space is locked onto a wrong action in a typical
option, and that is as true of the eighth option as of the first. The joint rate
behaves completely differently: the probability that a state is necrotic for
every option falls $59.0\rightarrow36.5\rightarrow15.5\rightarrow3.9$ percent,
and the expected steps of those same frozen policies fall
$959.5\rightarrow875.8\rightarrow509.4\rightarrow220.5$. Extra options do not
produce better options; they produce a set whose members fail in different
places. Table~\ref{tab:redundancy} says the same as a within-seed contrast, and
the learned-termination agent and its $\beta\equiv1$ match agree to within their
intervals, which is Section~\ref{sec:termination} in coverage terms.

\paragraph{How independent are the failures?}
If the $K$ necrotic sets were independent, the joint rate at $K{=}8$ would be
$0.594^{8}=1.5\%$ against the measured $3.9\%$; if they were identical, it would
be $59\%$. Solving $\Pr(\mathrm{D})=\Pr(D_o)^{\,\kappa}$ for $\kappa$ gives an
effective number of independent options: $\kappa=1.92$, $3.55$ and $6.20$ at
$K=2,4,8$. Options trained side by side, on shared trajectories with a shared
critic, fail together somewhat more than chance predicts and far less than a
common cause would, so a pool of $K$ options is worth roughly $K/2$ independent
attempts at covering the state space. That also says what limits the account:
the return on adding options is set by how decorrelated their failures are, a
property of the training procedure and not of the option count.

\paragraph{The bottleneck moves upstairs.}
Coverage that exists is not coverage that gets used. As D shrinks, the M class
grows in step: the router picks a necrotic option while a viable one is available
in $0\%$, $17.6\%$, $32.9\%$ and $42.7\%$ of states at $K=1,2,4,8$
(Figure~\ref{fig:coverage}a). By $K{=}8$ the dominant residual error is a routing
error and not a coverage error, which is the failure mode that interest functions
and attention over options were introduced to attack
\citep{khetarpal2020interest,chunduru2022attention}.

\paragraph{Redundancy has a price once the lower level works.}
The third row of Table~\ref{tab:redundancy} is the informative boundary case. An
agent trained with $\epslo=0.1$ still gets a coverage benefit from more options,
with joint necrosis falling by $35.5$ points, but its single-option rate rises by
$21.4$ points and its performance gets \emph{worse} by $134.9$ steps: a fixed
budget spread over eight options degrades each one, and the router must
distinguish eight policies instead of one. Redundancy earns its keep when there
is a failure to be redundant against, and charges rent when there is not.

Figure~\ref{fig:statemaps-full} in Appendix~\ref{app:extra} shows the transition
on the map itself. For the agent that never explored, joint
necrosis occupies $62$ of $103$ cells at $K{=}1$ and none at $K{=}8$, and four of
every five cells it vacates turn into routing misses instead of correct routes.
Its $\beta\equiv1$ match is indistinguishable from it, cell by cell. The agent
that explored starts at $37$ necrotic cells with one option and ends with $4$ at
$K{=}8$, while its correctly routed cells fall from $66$ to $50$ along the way.

\section{Related work}
\label{sec:related}

\paragraph{Options and option-critic.}
The options framework formalises temporally extended action
\citep{sutton1999between,precup2000temporal}, alongside earlier hierarchical
formulations \citep{parr1998ham,dietterich2000maxq}. Option-critic learns
intra-option policies and terminations from the task return
\citep{bacon2017option}; extensions cover continuous control
\citep{klissarov2017ppoc}, deeper hierarchies \citep{riemer2018abstract},
off-the-shelf policy optimisation \citep{zhang2019dac} and learned initiation
\citep{khetarpal2020interest,chunduru2022attention}, while other lines build
options from environment structure or intrinsic objectives
\citep{simsek2009betweenness,konidaris2009skillchaining,machado2017laplacian,machado2018eigenoption,eysenbach2019diayn,jinnai2019covertime,bagaria2020skillchaining}
or hierarchy over goals
\citep{kulkarni2016hdqn,vezhnevets2017feudal,nachum2018hiro}. Regret guarantees
for option-based learning exist under structural assumptions
\citep{drappo2024option}; Section~\ref{sec:theory} asks instead what a
return-trained termination rule can guarantee.

\paragraph{Termination collapse and its remedies.}
That option-critic terminations drift towards $1$ is documented and has motivated
direct remedies: a deliberation cost that charges for switching
\citep{harb2018waiting}, and a termination objective built on option models
instead of the control objective \citep{harutyunyan2019termination}. Those works
treat collapse as an empirical nuisance to engineer around.
Section~\ref{sec:theory} shows the drift is a property of the objective itself,
so those remedies act as stopgaps: they hold $\beta$ away from the boundary
without removing the reason it moves there.

\paragraph{Where hierarchy's benefits come from.}
\citet{nachum2019whyhierarchy} isolate the claimed benefits of goal-conditioned
hierarchies on continuous control and attribute most of the gain to improved
exploration in the high-level action space. Our result points the same way and
localises the mechanism one level lower, inside the option's own policy
parameterisation, where it is visible state by state.

\paragraph{Softmax policy gradient.}
\citet{mei2020global} and \citet{mei2020escaping} show that softmax-parameterised
policy gradient is initialisation-sensitive and slows near the corners of the
simplex, and \citet{agarwal2021theory} show that exploration or regularisation is
needed for global convergence; entropy regularisation is the standard practical
response \citep{williams1992reinforce,mnih2016a3c,haarnoja2018sac}. Policy
necrosis is that phenomenon at $\tau=10^{-3}$, where the committal step takes one
update, and what we add is that option-critic covers for it with options.

\paragraph{Empirical practice.}
Reported gains often survive poorly once controls and seed counts are taken
seriously
\citep{henderson2018matters,engstrom2020implementation,andrychowicz2021matters,agarwal2021precipice,patterson2024empirical}.
This paper is an instance in the hierarchical setting.

\section{Discussion}
\label{sec:discussion}

The three results compose into one account. A near-deterministic softmax with a
large effective step size commits states to arbitrary actions and stops updating
them, leaving most of the state space necrotic for any given option; the
diagnostic of Section~\ref{sec:diagnostics} locates that failure state by state,
and restoring lower-level exploration repairs it, at which point one option
solves the task. Extra options repair nothing individually. They supply policies
whose necrotic sets only partly overlap, so joint failure falls while individual
failure does not, performance follows the joint quantity, and the benefit goes
with the failure it was covering. Learned termination plays no part in any of
this, and Section~\ref{sec:theory} explains why it never could: the rule the
termination gradient converges to does nothing at all when its values agree with
the selector's, and blocks the selector's exploration when they do not. The option
count measures redundancy against a repairable defect, and the case for learned
temporal abstraction has to be made somewhere else.

\paragraph{A structural prior can be credited for a defect elsewhere.}
The broader reading is not about options. A structural component sits on a base
learner, and if the base learner has a repairable failure, the component is
credited with whatever that failure was costing. The option count is an unusually
legible example: individual policies stay broken, their failures decorrelate, and
the aggregate improves. The same accounting applies to any architectural prior
evaluated by a curve over its own capacity, and it is settled the same way, by
repairing the base learner first and asking whether the prior still earns its
keep. Redundancy also deserves naming as a mechanism: ensembles obtain it
deliberately \citep{osband2016bootstrapped} and option-critic stumbles into it,
so the lever is not the number of options but how decorrelated their failures
are.

\paragraph{Limitations.}
Two boundaries are worth stating. The diagnostic needs a notion of which actions
are acceptable in a state, which an exactly solvable environment supplies and a
large one does not, so elsewhere it requires a surrogate whose quality any
conclusion inherits. And the exchange rate between options and exploration is set
by $\kappa$, a property of how the options were trained together, so the point at
which multiplicity stops helping moves with the training procedure even though
the mechanism does not.

\paragraph{Future work.}
Make redundancy deliberate: diversity objectives should be judged by how much
they raise $\kappa$ at fixed $K$, a cheaper and sharper target than end-to-end
return. Carry the necrosis test past exactly solvable environments, either
against a reference policy or learned critic, or by dropping the oracle and
tracking the two quantities that make necrosis self-sealing, policy saturation
and the vanishing score-function norm. And ask which termination objectives
escape Section~\ref{sec:theory}: the obstruction applies to rules trained on
return, while deliberation costs and predictability-based termination
\citep{harb2018waiting,harutyunyan2019termination} change the objective, so the
same analysis should say whether either can be shown to help.

\subsection*{Reproducibility statement}

Section~\ref{sec:setup} states the evaluation contract and the hyperparameters,
Appendix~\ref{app:setup} lists the remaining settings for the tabular and neural
experiments, and Appendix~\ref{app:stats} gives the resampling scheme behind
every interval. Section~\ref{sec:diagnostics} gives both necrosis criteria in
full, including the tie-handling rule and the audit of the value-iteration action
sets against breadth-first search. Appendix~\ref{app:theory} contains the model,
the algorithms and complete proofs for Proposition~\ref{prop:vacuous},
Theorem~\ref{thm:separation} and Theorem~\ref{thm:ucbvi}. The accompanying code
release contains the training, checkpoint-evaluation, and analysis code for all
three experiment families, the derived per-seed and summary tables behind every
figure and table, a single entry point that rebuilds every figure in this paper
from those tables without touching a raw checkpoint, and a validator that checks
file hashes and the figure contract.

\subsection*{Ethics statement}

This work is a methodological study of a reinforcement learning algorithm. It involves no human subjects, no personal data,
and no deployment. We see no ethical concerns beyond the general question of how
empirical claims in the field are supported, which the paper addresses directly.

\subsection*{AI use statement}

We used generative AI tools for readability editing of the manuscript. We have not used generative AI
tools for proposing or refining the hypotheses, designing or critiquing the
methodology or experiments, implementing the methods, interpreting the results,
formulating the mathematical claims, or writing the proofs; generating synthetic
datasets and qualitative data analysis are not applicable to this work. We have
reviewed all AI-assisted work: every number in the paper was read back from the
released derived data files by an author, the figure-generation code was checked
against those files, and every theorem statement and proof was written and
verified by the authors. We take responsibility for the final content of this
work, including text, claims, and artifacts produced with the aid of generative
AI.

\bibliography{refs}
\bibliographystyle{iclr2027_conference}

\appendix

\section{Experimental settings}
\label{app:setup}

\paragraph{Tabular experiments.}
FourRooms is a $13\times13$ grid with $104$ free cells; the goal is the east
doorway, and the $103$ remaining cells form both the initial-state distribution
and the domain of the state-uniform diagnostics. Actions are the four cardinal
moves; a move into a wall leaves the agent in place. The intra-option policy is a
softmax over a $104\times4$ preference table with temperature $\tau=10^{-3}$ and
learning rate $0.25$; the termination function is a sigmoid over a $104$-vector
with learning rate $0.25$; the critic maintains tabular $Q_\Omega$ and $Q_U$ over
the same $104$ states with learning rate $0.5$; the policy over options is
$\epsup$-greedy on $Q_\Omega$ with $\epsup=0.1$. The discount is $0.99$. Each
cell of every grid runs seeds $0$ through $349$ for $1000$ episodes with a
$1000$-step cap. Five distinct kernels are obtained from local slip and global
teleport at $p\in\{0,\tfrac13,\tfrac12\}$, with $p=0$ shared by the two families.
The matched-fork experiment saves a checkpoint every $10$ episodes, giving $51$
prefix checkpoints and $51$ per branch, and snapshots the Mersenne Twister state
so that the fork is exact.

\paragraph{Neural experiments.}
The continuous environments place the agent at cell centres and use the same
$13\times13$ layout. In Direct, an action displaces the agent by one unit in a
cardinal direction with uniform noise in $[-0.1,0.1]$ per coordinate, and the
observation is $(x,y)$. In Inertia, an action changes the corresponding velocity
component by $0.1\pm0.01$, speed is clipped to $0.99$ per axis, position
integrates velocity, a collision zeroes it, and the observation is
$(x,y,v_x,v_y)$. The goal tolerance is half a cell. $Q_\Omega$, $Q_U$, the
policy, and the termination function are two-hidden-layer tanh networks of width
$64$, trained with Adam at learning rate $10^{-3}$, Polyak-averaged target
networks with coefficient $0.005$, a replay buffer of $10{,}000$ transitions,
batches of $64$, and one update every $4$ environment steps after $1000$ steps of
warm-up. Output biases are initialised so that the mean output over a grid of
states is zero. Each configuration runs $50$ seeds for $1000$ episodes with a
$1000$-step cap. Direct uses $\tau=10^{-3}$ for the base and $\eps$-greedy
configurations with $\epslo=0.1$, and $\tau=1$ with entropy coefficient $0.5$ for
the entropy configuration; Inertia uses $\tau=0.5$ throughout, with $\epslo=0.5$
and entropy coefficient $0.1$.

\section{Statistics}
\label{app:stats}

The tabular grids use $350$ training seeds per cell, matched across every
condition being contrasted, so all contrasts are paired within seed. Confidence
intervals are $95\%$ percentile intervals from $100{,}000$ paired bootstrap
resamples over seeds \citep{colas2018seeds,patterson2024empirical}, with fixed
bootstrap seeds recorded alongside the derived tables. The termination grid
retains block-level aggregates, so its resampling unit is a $35$-seed rank block,
ten per cell; the neural experiments resample $50$ individual seeds. Where the five
kernels are pooled we average within a seed first and bootstrap over seeds,
counting the duplicated $p=0$ cell once. The frozen evaluation is an exact
dynamic program and contributes no sampling error, so every reported interval
reflects training-seed variability alone.

\section{Termination rules and regret}
\label{app:theory}

This appendix supplies the model, the algorithms and the proofs behind
Section~\ref{sec:theory}. Two remarks frame what is and is not claimed. First,
``both levels use UCB'' does not by itself pin down an algorithm, so the analysis
fixes a minimal but non-trivial two-level model in which the interaction between
termination and exploration is exposed without any other moving part. Second,
the negative results are constructions inside that model, which is a special
case of the option setting, so they suffice to rule out a general guarantee.

\subsection{A two-level stochastic bandit}
\label{app:model}

Let $T$ be the number of atomic steps. The environment has a single recurring
state $s$. Let $\mathcal O=\{1,\dots,K\}$ be the option set and $\mathcal A_o$
the actions available inside option $o$, and write
$\mathcal L=\{(o,a):o\in\mathcal O,\ a\in\mathcal A_o\}$ for the set of
option--action leaves, with $M=|\mathcal L|$. Pulling leaf $i\in\mathcal L$
returns an independent reward in $[0,1]$ with mean $\mu_i$. Put
$\mu_o^\star=\max_{a}\mu_{o,a}$, $\mu^\star=\max_o\mu_o^\star$, and define the
upper, lower and total gaps
\[
\Delta_o^{\mathrm{up}}=\mu^\star-\mu_o^\star,
\qquad
\Delta_{o,a}^{\mathrm{lo}}=\mu_o^\star-\mu_{o,a},
\qquad
\Delta_{o,a}=\mu^\star-\mu_{o,a}=\Delta_o^{\mathrm{up}}+\Delta_{o,a}^{\mathrm{lo}}.
\]
With $I_t=(O_t,A_t)$ the leaf executed at step $t$ and $N_i(T)$ its visit count,
the pseudo-regret is
$\mathcal R_T=T\mu^\star-\mathbb E[\sum_{t\le T}\mu_{I_t}]
=\sum_i\Delta_i\,\mathbb E[N_i(T)]$.
Splitting by level, with $N_o(T)=\sum_aN_{o,a}(T)$,
\[
\mathcal R_T^{\mathrm{up}}=\sum_o\Delta_o^{\mathrm{up}}\mathbb E[N_o(T)],
\qquad
\mathcal R_T^{\mathrm{lo}}=\sum_{o,a}\Delta_{o,a}^{\mathrm{lo}}\mathbb E[N_{o,a}(T)].
\]

\begin{lemma}[Exact decomposition]
\label{lem:decomp}
For any selection algorithm and any $T$,
$\mathcal R_T=\mathcal R_T^{\mathrm{up}}+\mathcal R_T^{\mathrm{lo}}$.
\end{lemma}

\begin{proof}
Substitute $\Delta_{o,a}=\Delta_o^{\mathrm{up}}+\Delta_{o,a}^{\mathrm{lo}}$ into
$\mathcal R_T=\sum_{o,a}\Delta_{o,a}\mathbb E[N_{o,a}(T)]$ and use
$N_o(T)=\sum_aN_{o,a}(T)$ on the first summand.
\end{proof}

\subsection{Compatible two-level UCB, and the two termination rules}
\label{app:algos}

Each leaf is pulled once during initialisation. Afterwards, write
$b_i(t)=\widehat\mu_i(t-1)+\sqrt{2\ln t/N_i(t-1)}$ for the UCB1 index of leaf
$i$. The lower level proposes $a_o^+(t)\in\arg\max_{a\in\mathcal A_o}b_{o,a}(t)$
inside each option, the upper level scores options by
$B_o(t)=b_{o,a_o^+(t)}(t)=\max_{a}b_{o,a}(t)$ and proposes
$J_t\in\arg\max_o B_o(t)$. Ties are broken by a fixed deterministic rule. The two
levels are compatible in the sense that the option score is the maximum of its
own leaf indices, which is what makes UCB1's assumptions applicable.

Two algorithms are compared.
$\mathsf A_{\mathrm{fix}}$ fixes $\beta\equiv1$: at every step it sets
$O_t=J_t$ and executes $A_t=a_{O_t}^+(t)$.
$\mathsf A_{\mathrm{gi}}$ carries a current option $C_t$ into the step and
applies the greedy-interruption rule~\eqref{eq:hard}: if $\beta_t=1$ it sets
$O_t=J_t$, and if $\beta_t=0$ it sets $O_t=C_t$; either way it then executes
$A_t=a_{O_t}^+(t)$ and passes $C_{t+1}=O_t$ to the next step. Write
$\mathcal R_T^{\mathrm{fix}}$ and $\mathcal R_T^{\mathrm{gi}}$ for their
pseudo-regrets; $\mathcal R_T^{\mathrm{fix}}$ is the $\mathcal R_T^{\beta\equiv1}$
of the main text.

\subsection{Fixed \texorpdfstring{$\beta\equiv1$}{beta=1}: finite-time bounds}
\label{app:fix}

\begin{theorem}[Stepwise equivalence to leaf-level UCB1]
\label{thm:app-equiv}
Under $\mathsf A_{\mathrm{fix}}$, $I_t\in\arg\max_{i\in\mathcal L}b_i(t)$ at every
step, so $\mathsf A_{\mathrm{fix}}$ is UCB1 run directly on the $M$ leaves.
\end{theorem}

\begin{proof}
$b_{I_t}(t)=B_{J_t}(t)=\max_oB_o(t)=\max_o\max_ab_{o,a}(t)=\max_ib_i(t)$, and the
fixed tie-breaking rule selects the same leaf in both formulations.
\end{proof}

\begin{theorem}[Instance-dependent bound]
\label{thm:app-inst}
Let $C_0=1+\pi^2/3$. For rewards in $[0,1]$ and any $T\ge M$,
\[
\mathcal R_T^{\mathrm{fix}}\;\le\;
\sum_{(o,a):\Delta_{o,a}>0}\Bigl(\frac{8\ln T}{\Delta_{o,a}}+C_0\Delta_{o,a}\Bigr),
\]
and, level by level,
$\mathcal R_T^{\mathrm{up,fix}}\le\sum_{(o,a):\Delta_{o,a}>0}
\Delta_o^{\mathrm{up}}(8\ln T/\Delta_{o,a}^2+C_0)$ and
$\mathcal R_T^{\mathrm{lo,fix}}\le\sum_{(o,a):\Delta_{o,a}>0}
\Delta_{o,a}^{\mathrm{lo}}(8\ln T/\Delta_{o,a}^2+C_0)$.
\end{theorem}

\begin{proof}
UCB1 guarantees $\mathbb E[N_i(T)]\le 8\ln T/\Delta_i^2+C_0$ for every suboptimal
arm \citep{auer2002ucb}. By Theorem~\ref{thm:app-equiv} the arms are the leaves,
so multiplying by $\Delta_i$ and summing gives the first display. Multiplying the
same visit-count bound by $\Delta_o^{\mathrm{up}}$ and by
$\Delta_{o,a}^{\mathrm{lo}}$ and applying Lemma~\ref{lem:decomp} gives the other
two.
\end{proof}

\begin{corollary}[Distribution-free bound]
\label{cor:app-df}
For $T\ge\max\{M,2\}$,
$\mathcal R_T^{\mathrm{fix}}\le\min\{T,\;4\sqrt{2MT\ln T}+C_0M\}$.
\end{corollary}

\begin{proof}
Split the leaves at a threshold $\eps\in(0,1]$. Leaves with $\Delta_i\le\eps$
contribute at most $T\eps$; the rest contribute at most $8M\ln T/\eps+C_0M$ by
Theorem~\ref{thm:app-inst}. Choosing $\eps=\sqrt{8M\ln T/T}$ makes the first two
terms $4\sqrt{2MT\ln T}$. If that $\eps$ exceeds $1$ the bound already exceeds
the trivial bound $T$, which always holds because single-step gaps are at most
one.
\end{proof}

A matching lower bound is classical: for Bernoulli leaves with a unique optimum
and any consistent algorithm, $\liminf_T\mathcal R_T/\ln T\ge
\sum_{i:\Delta_i>0}\Delta_i/\operatorname{kl}(\mu_i,\mu^\star)$
\citep{lai1985asymptotically}, so the $O(\ln T)$ rate of
Theorem~\ref{thm:app-inst} cannot be improved to $o(\ln T)$ on a fixed
non-degenerate instance.

\subsection{Two value settings for the return-maximising rule}
\label{app:valuesetting}

Rule~\eqref{eq:hard} compares the value of continuing with the value of handing
control back, and everything turns on which estimates those two are.

\paragraph{Optimistic indices.}
Set $Q_t^+(s,o)=B_o(t)$ and let the upper policy be the point mass on $J_t$, so
$V_t^+(s)=B_{J_t}(t)=\max_oB_o(t)$. Proposition~\ref{prop:vacuous} of the main
text is then immediate.

\begin{proof}[Proof of Proposition~\ref{prop:vacuous}]
For any current option $C_t$ we have
$Q_t^+(s,C_t)=B_{C_t}(t)\le\max_oB_o(t)=V_t^+(s)$, so \eqref{eq:hard} selects
$\beta_t=1$ and the step executes $J_t$. Induct on $t$: the two algorithms start
identically; if the histories before step $t$ agree then all counts, sample means
and indices agree, hence $J_t$ agrees, and both algorithms execute $J_t$ with the
same lower-level action, so the histories after step $t$ agree. Action sequences,
and therefore regret, coincide on every sample path.
\end{proof}

\paragraph{Raw-value indices.}
Option-critic's state value is the upper policy's expectation of raw critic
values, while a confidence bonus is a device the exploration rule adds on top. In
the present model that setting is
\[
\widehat q_o(t)=\widehat\mu_{o,a_o^+(t)}(t-1),
\qquad
Q_t(s,o)=\widehat q_o(t),
\qquad
V_t(s)=\widehat q_{J_t}(t),
\tag{OC-UCB}
\]
so the option the selector proposes need not be the one with the largest raw
value. When $B_{J_t}(t)>B_{C_t}(t)$ and $Q_t(s,J_t)<Q_t(s,C_t)$ hold at once, the
selector asks to explore $J_t$ and the termination test refuses.

\subsection{Proof of Theorem~\ref{thm:separation}}
\label{app:sep}

\paragraph{Linear lower bound.}
Take two options. Both actions of the good option $G$ return
$\operatorname{Ber}(3/4)$; both actions of the bad option $B$ return exactly
$1/2$.
Then $\mu_G^\star=3/4$, $\mu_B^\star=1/2$, $\Delta_B^{\mathrm{up}}=1/4$, and every
lower-level gap is zero, so any regret is caused by the option choice alone.
Initialise in the order $(G,1),(G,2),(B,1),(B,2)$ without applying the
termination test, so the current option after step four is $B$. Let $E$ be the
event that both initialisation draws from $G$ return $0$, which has probability
$(1/4)^2=1/16$.

On $E$ both sample means of $G$ are $0$ while both of $B$ are $1/2$. We show by
induction that the algorithm executes $B$ forever. Assume the current option
before step $t\ge5$ is $B$ and $G$ has received no further samples. If $J_t=B$
then $Q_t(s,C_t)=1/2=V_t(s)$, and the convention that ties terminate gives
$\beta_t=1$, after which the reselection returns $B$. If $J_t=G$ then
$Q_t(s,C_t)=1/2>0=Q_t(s,G)=V_t(s)$, so $\beta_t=0$ and the step continues with
$B$. Either way $B$ executes and the hypothesis is preserved. On $E$ the bad
option therefore runs on all $T-2$ steps from step three onwards, each at
pseudo-regret $1/4$, so
$\mathcal R_T^{\mathrm{gi}}\ge\frac1{16}\cdot\frac14(T-2)=(T-2)/64$. Since
single-step gaps are at most one, $\mathcal R_T^{\mathrm{gi}}\le T$ always, and
the worst case over the class is $\Theta(T)$.

\paragraph{Logarithmic upper bound on the same instance.}
That instance has two optimal leaves at $3/4$ and two suboptimal leaves at $1/2$,
each with total gap $1/4$. Theorem~\ref{thm:app-inst} gives
$\mathcal R_T^{\mathrm{fix}}\le2(8\ln T/(1/4)+C_0/4)=64\ln T+C_0/2$. Hence
$\mathcal R_T^{\mathrm{fix}}/\mathcal R_T^{\mathrm{gi}}\to0$. The class-level
statement combines this with Corollary~\ref{cor:app-df}.

\paragraph{Episode resets do not repair it.}
The continuing counterexample is broken by an unconditional reset of the option
at every episode boundary, but resets do not restore a sublinear rate in general.
Take a two-step episodic MDP. At the first step, state $s_0$ admits only the bad
option $B$; its single action returns $0$ and moves deterministically to $s_1$, so
the current option on arrival at $s_1$ is always $B$. At $s_1$ both options are
available with two actions each; both actions of $B$ return exactly $1/4$, both
actions of $G$ return $\operatorname{Ber}(3/4)$, and the episode ends. The optimal
episode value is $3/4$ and choosing $B$ at $s_1$ costs $1/2$. Initialise the four
leaves at $s_1$ over the first four episodes without applying the test, which
already costs $1$, and let $E$ be as before. On $E$ the critic values at $s_1$
are $0$ for both actions of $G$ and $1/4$ for both actions of $B$; since the
fixed prefix rebuilds $C=B$ at $s_1$ in every episode, the argument above applies
unchanged and $B$ executes at $s_1$ in all remaining $N-4$ episodes. Hence
$\mathbb E\operatorname{Reg}_N^{\mathrm{gi}}\ge1+(N-4)/32=1+(T-8)/64$ with
$T=2N$, while $\beta\equiv1$ at $s_1$ is UCB1 over four leaves with gaps $1/2$
and satisfies $\mathbb E\operatorname{Reg}_N^{\beta\equiv1}\le32\ln N+C_0$.
Removing this family of counterexamples needs a condition stronger than resetting
the option, for instance a guarantee that every available state--option pair is
actually executed and not merely proposed.

\subsection{No pointwise dominance}
\label{app:reverse}

Theorem~\ref{thm:separation} is a worst-case separation, and
the closing paragraph of Section~\ref{sec:theory} records that it cannot be
upgraded. Consider a
deterministic instance with $\mu_{G,1}=\mu_{G,2}=1$ and
$\mu_{B,1}=\mu_{B,2}=0$, initialised in the order $(B,1),(B,2),(G,1),(G,2)$, so
the current option after initialisation is $G$ with $\widehat q_G=1$ and
$\widehat q_B=0$. Under the raw-value indices, a proposal of $G$ gives
$Q_t(s,C_t)=V_t(s)$ and terminates into $G$ again, and a proposal of $B$ gives
$Q_t(s,C_t)=1>0=V_t(s)$ and continues with $G$; so $\mathsf A_{\mathrm{gi}}$
executes $G$ for ever and $\mathcal R_T^{\mathrm{gi}}=2$ for all $T\ge4$. Under
$\beta\equiv1$ the leaf indices force a return to $B$: at $t=11$,
$B_{B}(11)=\sqrt{2\ln 11}>1+\sqrt{2\ln 11/4}=B_G(11)$, so
$\mathcal R_{11}^{\mathrm{fix}}=3>2=\mathcal R_{11}^{\mathrm{gi}}$. The
excursions recur, because if $\beta\equiv1$ pulled $B$ only finitely often its
index would grow like $\sqrt{2\ln t/m}\to\infty$ while $B_G(t)\to1$. The veto in
\eqref{eq:hard} therefore has two faces: it locks in an error when the current
option is suboptimal and the unexplored option is not, and it saves exploration
cost when the current option is already optimal.

\subsection{Relation to the interruption theorem}
\label{app:interruption}

\citet{sutton1999between} show that interrupting an option whenever continuing is
worth less than reselecting under a \emph{fixed} Markov policy $\mu$ improves the
value everywhere and strictly somewhere. Two hypotheses of that theorem fail for
the online process analysed here. The quantities $Q_t,V_t$ are finite-sample
estimates and not exact values of a fixed policy, and the upper policy changes
with the counts, so there is no fixed $\mu$ to improve upon. In the
counterexample, $G$ is genuinely worth $3/4$, an unlucky initialisation makes its
estimate $0$, and the termination rule reads that estimation error as permanent
evidence for continuing with $B$ while simultaneously preventing the samples that
would correct it. The classical theorem concerns policy improvement once values
are known; it makes no claim about this estimation-visitation feedback loop.

\subsection{Episodic tabular RL}
\label{app:rl}

Let $\mathcal S$ be a finite state set with $S=|\mathcal S|$, let $H$ be the
episode length, $N$ the number of episodes and $T=NH$. Write
$\mathcal U(s)=\{(o,a):o\in\mathcal O(s),a\in\mathcal A_o(s)\}$ for the augmented
actions at $s$ and $D=\max_s|\mathcal U(s)|$. A Bellman-consistent optimistic
algorithm maintains $\overline Q_{k,h}(s,o,a)$ from the optimistic Bellman
recursion of a single augmented MDP, and selects
$a_{k,h}^+(s,o)\in\arg\max_a\overline Q_{k,h}(s,o,a)$ below and
$o_{k,h}^+(s)\in\arg\max_o\overline Q_{k,h}(s,o,a_{k,h}^+(s,o))$ above.

\begin{theorem}[Augmented isomorphism]
\label{thm:app-iso}
Under $\beta\equiv1$ the executed pair satisfies
$u_{k,h}\in\arg\max_{u\in\mathcal U(s_{k,h})}\overline Q_{k,h}(s_{k,h},u)$, and
the hierarchical agent produces exactly the same states, actions, rewards and
successor states as the same optimistic algorithm run on
$\mathcal M^+=(\mathcal S,\mathcal U,P,r,H)$.
\end{theorem}

\begin{proof}
Nested maximisation is joint maximisation:
$\overline Q(s,o^+,a^+(s,o^+))=\max_o\max_a\overline Q(s,o,a)
=\max_{u}\overline Q(s,u)$. With $\beta\equiv1$ the option is resampled at every
step, so no constraint carries over from the previous step and the pair $u$ is
free. Executing $u$ yields reward $r(s,u)$ and successor $P(\cdot\,|\,s,u)$ by
definition of $\mathcal M^+$, so the observation histories agree, and induction
propagates the agreement to all counts and optimistic values.
\end{proof}

Theorem~\ref{thm:ucbvi} follows by instantiating the optimistic algorithm as
UCBVI-BF and quoting its guarantee.

\begin{proof}[Proof of Theorem~\ref{thm:ucbvi}]
Assume known deterministic rewards in $[0,1]$, unknown transitions, a common
augmented action set of size $D$, and failure probability $\delta$. By
Theorem~\ref{thm:app-iso} the hierarchical agent and UCBVI-BF on $\mathcal M^+$
have identical episode regret. Theorem 2 of \citet{azar2017minimax} gives, with
probability at least $1-\delta$ and $L=\ln(5HSDT/\delta)$,
$\operatorname{Reg}_N\le30HL\sqrt{SDN}+2500H^2S^2DL^2+4H^{3/2}\sqrt{NL}$.
Substituting $N=T/H$ turns the first term into $30L\sqrt{HSDT}$ and the third
into $4H\sqrt{TL}$.
\end{proof}

Two caveats belong with this bound. If rewards and transitions depend only on the
primitive action, distinct pairs $(o,a)$ duplicate one another, and taking
$D\le|\mathcal O||\mathcal A|$ is correct but loose unless the algorithm pools
statistics across options. And UCBVI is model-based: a bonus at each level is not
enough to invoke the theorem, since the optimistic values must satisfy the
corresponding Bellman recursion and confidence event.

\subsection{Summary}
\label{app:theory-summary}

Consider any tabular scheme whose two-level selection is equivalent to
maximisation over augmented actions. If the behaviour selector and the
termination test share one value function and the state value is the maximum over
option values, then $\beta\equiv1$ holds pathwise, the two algorithms coincide,
and both inherit whichever guarantee the underlying optimistic algorithm carries.
If the selector uses optimistic values while the termination test uses raw
values, no uniform sublinear guarantee survives, even at $H=2$ with per-episode
option resets. The two rules are not comparable instance by instance. What can be
asserted, and what Section~\ref{sec:theory} asserts, is that a termination rule
trained to maximise return is redundant in the consistent case and unguaranteed
in the standard case.

\section{Additional results}
\label{app:extra}

\begin{table}[h]
\centering
\caption{\textbf{Learned termination and forced per-step reselection are
indistinguishable under exact frozen evaluation.} Shortest-path FourRooms,
episode-1000 checkpoints, expected capped transitions and success probability
under horizon-1000 exact dynamic programming with $\epslo=0$, averaged over the
five kernels within seed. Differences are paired over $350$ matched seeds with
$95\%$ bootstrap intervals; negative favours learned termination. At $K{=}1$
termination cannot change which option runs, so that row measures the noise floor
of the seed matching.}
\label{tab:termination-frozen}
\small
\setlength{\tabcolsep}{5pt}
\begin{tabular}{ccccccc}
\toprule
& \multicolumn{2}{c}{Expected steps} & \multicolumn{2}{c}{Success probability} & \multicolumn{2}{c}{Paired difference in steps} \\
\cmidrule(lr){2-3}\cmidrule(lr){4-5}\cmidrule(lr){6-7}
$K$ & learned OC & $\beta\equiv1$ & learned OC & $\beta\equiv1$ & mean & 95\% CI \\
\midrule
1 & 962.49 & 963.21 & 0.038 & 0.037 & $-0.71$ & $[-2.62,\ 1.19]$ \\
2 & 821.44 & 822.31 & 0.195 & 0.194 & $-0.87$ & $[-7.93,\ 6.20]$ \\
4 & 377.20 & 384.37 & 0.687 & 0.680 & $-7.17$ & $[-21.34,\ 6.95]$ \\
8 & 117.35 & 122.46 & 0.964 & 0.959 & $-5.11$ & $[-12.77,\ 2.57]$ \\
\bottomrule
\end{tabular}
\end{table}

\begin{figure}[h]
\centering
\includegraphics[width=\textwidth]{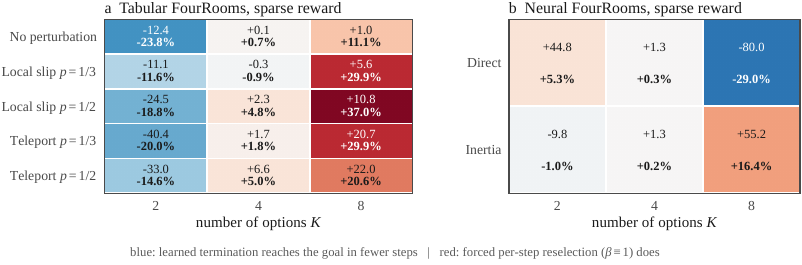}
\caption{\textbf{Table~\ref{tab:termination} as a heat map.} Last-100-episode
step count of learned termination minus that of a matched $\beta\equiv1$ agent,
in absolute steps (top of each cell) and relative to the control (bottom).
Blue favours learned termination and red favours $\beta\equiv1$. The tabular
grid reverses sign completely between $K{=}2$ and $K{=}8$; the neural grid has no
consistent direction.}
\label{fig:termination}
\end{figure}

\begin{figure}[h]
\centering
\includegraphics[width=\textwidth]{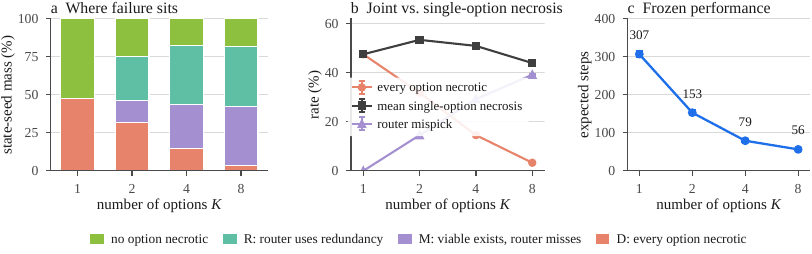}
\caption{\textbf{The sparse-reward replication of Figure~\ref{fig:coverage}.}
Episode-1000 control checkpoints of the matched-fork grid, five-kernel average
over $350$ matched seeds, with necrosis defined by the $1\%$ optimal-action mass
criterion of Section~\ref{sec:diagnostics}. Joint necrosis falls
$47.5\rightarrow31.8\rightarrow14.6\rightarrow3.3$ percent while the mean
single-option rate stays near $50\%$, and the router's miss rate rises to
$39.2\%$; expected steps fall $306.7\rightarrow152.9\rightarrow79.1\rightarrow
56.4$. The effective independent option count is $\kappa=1.82$, $2.85$ and
$4.14$ at $K=2,4,8$.}
\label{fig:sparse-coverage}
\end{figure}

\begin{figure}[h]
\centering
\includegraphics[width=\textwidth]{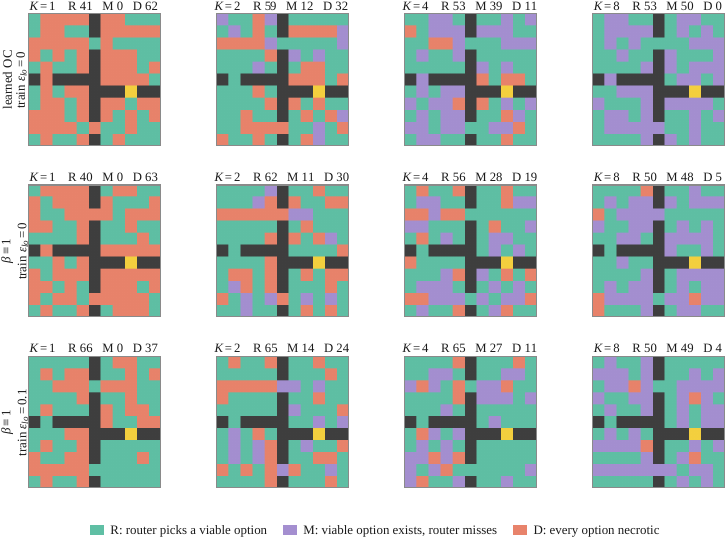}
\caption{\textbf{Where coverage appears, and where the router loses it.} Every
non-goal cell of constant-cost FourRooms under teleport $p=\tfrac12$, at the
episode-1000 checkpoint of the spatial-medoid seed, chosen by minimum Hamming
distance across all twelve panels without reference to performance. Dark grey is
wall and yellow is the goal. The middle row is a $\beta\equiv1$ agent trained
without lower-level exploration; its maps are visually and numerically close to
the learned-termination agent above it at every option count, which is the
coverage-level statement of Section~\ref{sec:termination}.}
\label{fig:statemaps-full}
\end{figure}

\end{document}